\documentclass[11pt]{article}
\usepackage[
    letterpaper,
    top=0.75in,
    bottom=1in,
    left=0.625in,
    right=0.625in,
    heightrounded ]{geometry}
\usepackage{amsmath,amssymb,amsfonts,amsthm}
\usepackage{algorithmic}
\usepackage{graphicx}
\usepackage{algorithm,algorithmic}
\usepackage{hyperref}
\hypersetup{hidelinks=true}
\usepackage{textcomp}
\usepackage{booktabs}
\usepackage{threeparttable}
\usepackage{multirow}
\usepackage{makecell}
\usepackage{siunitx}

\usepackage{makecell}
\newtheorem{proposition}{Proposition}
\def\b{\ensuremath\boldsymbol}

\usepackage[blocks]{authblk}

\title{\textbf{LoRSA: \\Toward Generalizable Parameter-Efficient Fine-Tuning for Biomedical Downstream Tasks}}

\author[1,2]{\small\textbf{Saed Moradi}}
\author[3]{\textbf{Benyamin Ghojogh}\thanks{These authors contributed equally to this work.}}
\author[2]{\textbf{ M. Hadi Sepanj}\protect\footnotemark[1]}
\author[4]{\textbf{Yimin Yang}}
\author[1]{\textbf{Ashirbani Saha}}

\affil[1]{Department of Oncology, McMaster University, Hamilton, Canada}
\affil[2]{ Department of Systems Design Engineering, University of Waterloo, Waterloo, Canada}
\affil[3]{Department of Electrical and Computer Engineering, University of Waterloo, Waterloo, Canada.}
\affil[4]{Department of Electrical and Computer Engineering, Western University, London, Canada}

\date{}

\newcommand{\keywords}[1]{\noindent\textbf{Keywords:} #1}
\begin{document}

\maketitle

\begin{abstract}
Parameter-efficient fine-tuning enables the adaptation of vision
foundation models to biomedical tasks under limited computational
resources, but a single low-rank update can constrain all task-specific
changes to one narrow parameter subspace. This restriction may prevent
the model from simultaneously representing globally shared task
structure and localized residual directions required for generalization
to unseen imaging domains. We introduce LoRSA, a global--residual
adaptation framework that jointly learns a dense low-rank component
and a dynamically structured-sparse low-rank component. The dense
component captures globally coordinated task adaptation, while the
structured component provides complementary residual corrections whose
support evolves during training. We characterize the representational
capacity, approximation properties, rank structure, and singular-subspace
complementarity of this decomposition. We evaluate LoRSA for four-class
breast-density classification using DINOv3-Base, with VinDr-Mammo as
the source domain and MammosighTR and RSNA as unseen external domains.
LoRSA remains competitive on the internal validation set and achieves
the best external macro-F1 on both target datasets, improving upon the
strongest competing method by 2.15 percentage points on MammosighTR
and 3.09 percentage points on RSNA. Weight-matrix analysis further
shows that approximately $92\%$ of the energy of each adaptation
component lies outside the bilateral singular subspace of the other,
indicating that the two components learn largely complementary update
directions. These results suggest that organizing adaptation capacity
into distinct global and residual paths can improve the external-domain
generalization of parameter-efficiently adapted biomedical vision
models.
\end{abstract}

\vspace{0.5em}


\vspace{2em}

\section{Introduction}

Large pre-trained Vision Foundation Models (VFMs) \cite{caron2021emerging,kirillov2023segment,radford2021learning} have led to great success in computer vision. However, fine-tuning these models  specifically for medical image analysis is challenging because downstream datasets are often small and annotated target data is limited \cite{tajbakhsh2016convolutional,zhou2021review}. 
Furthermore, external test cohorts frequently differ from the training distribution due to variations in scanner vendor, acquisition protocol, preprocessing pipelines, and patient demographics, a phenomenon known as domain shift \cite{castro2020causality,guan2021domain}. 
Additionally, model evaluation and convergence can be hindered by inter-observer variability and potential subjectivity in ground truth labels \cite{shi2024survey}. 
 In a specific downstream task, these factors create a practical domain-generalization problem: 
a VFM adapted on a small dataset should learn the essential  task-specific representation that remains reliable on unseen external datasets while preserving the transferable knowledge of the pre-trained VFM \cite{fu2023effectiveness}.

Parameter-efficient fine-tuning (PEFT) \cite{xu2026parameter,xin2026parameter,wang2025parameter} emerged to address the main challenge of resource limitation in full fine-tuning of VFMs. However, a key difficulty is that downstream  PEFT can unintentionally reduce, rather than exploit, the broad representational capacity of the adapted VFMs \cite{ni2024pace,nguyen2024saft,shu2023clipood}. Recent geometry-aware domain-generalization work observes that when VFMs are fine-tuned for specific tasks, their rich pre-trained expressive capacity is often compressed into narrow, domain-specific subspaces. This compression causes the adapted VFM to overfit to the statistics and decision boundaries of the downstream task’s source domain, producing domain bias (poor transfer under appearance or acquisition changes) \cite{zang2026geco}. This prohibits the learning of a generalized representation of the downstream task and limits its performance on the unseen target domain (external datasets) related to the downstream task. This observation is central to our problem setting: Biomedical image analysis, breast density classification in mammograms as a case study in this work, requires task-specific adaptation without collapsing the broad pretrained representation into a narrow, source-domain-specific subspace.

Low-rank PEFT methods \cite{hu2022lora,zhang2023adalora,hayou2024lora+,liu2024dora,hu2023structure}  are particularly effective when downstream adaptation can be expressed within a compact parameter subspace, especially when part of this subspace is reusable across layers or model components \cite{renduchintala2024tied,kaushik2025universal}. In biomedical imaging, such shared structure may arise from recurring anatomical and morphological patterns across datasets, enabling generalist models to transfer representations across imaging domains, anatomical targets, and clinical tasks \cite{ouyang2024towards,ma2024segment}. However, external-domain generalization may require more than a single low-rank adaptation. Such a   purely low-rank subspace may not be global enough to capture an exhaustive set of all localized, rare,  and domain-sensitive residual directions which can be induced by  potential domain shifts. While adapting based on a single training dataset, the residual directions are forced into the low-rank subspace, which limits the adapted VFM to source-specific artifacts, thereby increasing the risk of training dataset bias and overfitting.


This motivates the central problem considered in this work: how can we learn an adaptation that captures both the global low-dimensional structure of the downstream task and the  residual variations  needed for external generalization, without collapsing the VFM representation into a narrow source-specific subspace for the biomedical image analysis down-stream task? 

The main contributions of this work are twofold:
\begin{enumerate}
    \item We identify a key limitation of low-rank-only adaptation for domain generalization. A single low-rank update can constrain adaptation to a limited subspace and fail to capture structured residual directions needed for transfer to unseen medical imaging domains. Our theoretical and empirical results show that neither increasing rank nor sparsifying the same low-rank update fully resolves this limitation.
    \item We introduce a complementary global–residual adaptation framework. The proposed approach jointly learns a dense low-rank component for global task adaptation and a dynamically structured residual component for complementary directions outside the dominant low-rank subspace. We characterize its approximation capacity, spectral structure, and subspace complementarity, and validate these properties through matrix weight-level analysis and evaluation on two external datasets. To the best of our knowledge, this is the first theoretically motivated and empirically supported weight-space study of how PEFT structure affects out-of-distribution generalization in medical imaging.
\end{enumerate}

\section{Background}

PEFT methods have become the dominant strategy for adapting large pre-trained models to downstream tasks without updating the full parameter set. Among them, low-rank adaptation (LoRA) and its variants have become prominent PEFT techniques.

\subsection{LoRA}

Among PEFT methods, Low-Rank Adaptation (LoRA) is one of the most influential and widely adopted approaches because it provides a simple, mergeable, and architecture-preserving mechanism for task adaptation \cite{hu2022lora}. Given a frozen pre-trained weight matrix $\b{W}_0 \in \mathbb{R}^{d_{\mathrm{out}} \times d_{\mathrm{in}}}$, LoRA constrains the task-specific update $\Delta \b{W}$ to a low-rank factorization:
\begin{align}
    &\b{W} = \b{W}_0 + \Delta \b{W},
    \quad
    \Delta \b{W} = \b{BA}, \label{eq:lora_update}
    \\
    &\b{B} \in \mathbb{R}^{d_{\mathrm{out}} \times r},\;
    \b{A} \in \mathbb{R}^{r \times d_{\mathrm{in}}},
    \;
    r \ll \min(d_{\mathrm{out}}, d_{\mathrm{in}}).
\end{align}
For an input representation $x$, the adapted forward pass becomes:
\begin{equation}
    \b{h} = \b{W}_0\b{x} + \frac{\alpha}{r}\b{BAx},
    \label{eq:lora_forward}
\end{equation}
where $\alpha$ is a scaling factor and only the low-rank factors $\b{A}$ and $\b{B}$ are trained while $\b{W}_0$ remains frozen. This reduces the number of trainable parameters from $d_{\mathrm{out}}d_{\mathrm{in}}$ to $r(d_{\mathrm{out}}+d_{\mathrm{in}})$ for each adapted layer, while allowing the update to be merged into the base weight at inference time. 
The intuition behind LoRA is that full fine-tuning gives every entry of $\b{W}_0$ a degree of freedom, whereas LoRA forces the update to pass through a narrow $r$-dimensional bottleneck. In this view, $\b{A}$ first projects the input into a compact task-specific subspace, and $\b{B}$ then lifts this low-dimensional correction back to the original feature dimension. Despite its efficiency, LoRA also  introduces an important limitation. The low-rank-only assumption can be too restrictive for complex downstream adaptation.

\subsection{RoseLoRA}

RoseLoRA extends the LoRA family by introducing sparsity into the effective LoRA update \cite{wang2024roselora}. The motivation is that some downstream tasks, especially knowledge editing, require precise modifications to a limited subset of model parameters rather than dense changes to the entire weight matrix. Starting from the standard LoRA parameterization in Eq.~\eqref{eq:lora_update}, RoseLoRA aims to constrain the product $\b{BA}$ itself to be sparse:
\begin{equation}
    \min_{\b{A},\b{B}} \; \mathcal{L}(\mathcal{D}; \b{W}_0 + \b{BA})
    \quad
    \mathrm{s.t.}
    \quad
    \frac{\|\b{BA}\|_0}{d_{\mathrm{out}}d_{\mathrm{in}}} \leq \tau,
    \label{eq:roselora_product_constraint}
\end{equation}
where $\tau$ controls the allowed density of the effective update. The main challenge in RoseLoRA is that while the sparsity constraint on the weight updates might be useful for applications like knowledge editing, it is still a low-rank adaptation and is too simplistic for complex downstream tasks.

\subsection{RoSA}

RoSA takes a different direction by explicitly combining low-rank and sparse adaptation in parallel \cite{nikdan2024rosa}. Its motivation comes from Robust Principal Component Analysis (RPCA) \cite{candes2011robust}, where a matrix is decomposed into a low-rank component and a sparse component. RoSA applies this idea to PEFT by assuming that the full fine-tuning update can be better approximated as:
\begin{equation}
    \Delta \b{W} \approx \Delta \b{W}_L + \Delta \b{W}_S,
    \label{eq:rosa_decomp_general}
\end{equation}
where $\Delta \b{W}_L$ captures the dominant low-rank adaptation and $\Delta \b{W}_S$ captures sparse residual directions. For each adapted weight matrix, RoSA uses:
\begin{equation}
    \Delta \b{W}_L = \b{BA},
    \quad
    \Delta \b{W}_S = \b{M} \odot \b{S},
    \label{eq:rosa_update}
\end{equation}
where $\b{M} \in \{0,1\}^{d_{\mathrm{out}} \times d_{\mathrm{in}}}$ is a binary sparsity mask, $\b{S}$ contains the trainable sparse values, and $\odot$ denotes Hadamard product. The corresponding optimization can be written as:
\begin{equation}
    \min_{\b{A},\b{B},\b{S}} \; \mathcal{L}\left(\mathcal{D}; \b{W}_0 + \b{BA} + \b{M} \odot \b{S}\right),
    \label{eq:rosa_objective}
\end{equation}
with $\b{W}_0$ frozen and only the low-rank factors and sparse values trained.
While RoSA has  the key advantage of  preserving the strengths of LoRA while addressing the low-rank bottleneck, the sparse mask $\b{M}$ is generated using a small calibration or mask-generation subset before the main adapter training. In other words, the mask $\b{M}$ for the location of sparse values is fixed during training. 

\section{LoRSA: Low-Rank and Structured-Sparse Adaptation}

\subsection{Methodology}

Let $\b{W}_0 \in \mathbb{R}^{d_{\mathrm{out}} \times d_{\mathrm{in}}}$
denote a pretrained weight matrix. Standard parameter-efficient fine-tuning
constructs an adapted weight matrix as:
\begin{align}
    \b{W} = \b{W}_0 + \Delta \b{W},
\end{align}
where $\Delta \b{W}$ is the task-specific parameter update.

Motivated by the low-rank-plus-sparse decomposition in RPCA \cite{candes2011robust}, we model the task-specific update as:
\begin{align}
    \Delta \b{W} = \b{L} + \b{S},
    \label{equation_low_rank_sparse_decomposition}
\end{align}
where $\b{L}$ is a dense low-rank component and $\b{S}$ is a structured sparse low-rank component. 

Different approaches can be used to implement the low-rank component and structured sparse component. 
For example, these components can be parameterized using LoRA \cite{hu2022lora} and RoseLoRA \cite{wang2024roselora}, respectively:
\begin{align}
    \b{L}
    &:=
    \frac{\alpha_l}{r_l}\b{B}_l\b{A}_l,
    \label{eq:lora_component}\\
    \b{S}
    &:=
    \frac{\alpha_s}{r_s}\b{B}_s\b{A}_s,
    \label{eq:roselora_component}
\end{align}
where $r_l, r_s \ll \min\{d_{\mathrm{in}}, d_{\mathrm{out}}\}$ are the ranks of the $\b{L}$ and $\b{S}$ components, and the dimensionalities of the matrices are:
\begin{align*}
    &\b{A}_l
    \in
    \mathbb{R}^{r_l \times d_{\mathrm{in}}},
    \quad
    \b{B}_l
    \in
    \mathbb{R}^{d_{\mathrm{out}} \times r_l},\\
    &\b{A}_s
    \in
    \mathbb{R}^{r_s \times d_{\mathrm{in}}},
    \quad
    \b{B}_s
    \in
    \mathbb{R}^{d_{\mathrm{out}} \times r_s}.
\end{align*}
The resulting adapted layer is therefore:
\begin{align}
    \b{W}
    &=
    \b{W}_0
    +
    \frac{\alpha_l}{r_l}\b{B}_l\b{A}_l
    +
    \frac{\alpha_s}{r_s}\b{B}_s\b{A}_s.
    \label{eq:combined_adapter}
\end{align}

Note that, in contrast to RoSA, the locations of the nonzero entries in the structured sparse component of LoRSA are not fixed during training. 

\subsection{Discussion on the Rank and Sparsity in LoRSA}

The low-rank component satisfies:
\begin{align}
    \operatorname{rank}(\b{L}) \leq r_l,
\end{align}
and is generally dense. It can therefore represent globally coordinated changes across many input and output dimensions. 

In contrast, the structured sparse component imposes row-wise sparsity on $\b{A}_s$ and column-wise sparsity on
$\b{B}_s$. For every $k \in \{1,\ldots,r_s\}$, we impose:
\begin{align}
    &\left\|(\b{A}_s)_{k:}\right\|_0
    \leq
    q_a d_{\mathrm{in}}, \\
    &\left\|(\b{B}_s)_{:k}\right\|_0
    \leq
    q_b d_{\mathrm{out}},
\end{align}
where $q_a,q_b \in [0,1]$ denote the fractions of entries retained in
each row and column, respectively.

If every row of $\b{A}_s$ retains at most a fraction $q_a$ of its entries
and every column of $\b{B}_s$ retains at most a fraction $q_b$, then each
rank-one component contains at most $q_aq_b d_{\mathrm{in}}d_{\mathrm{out}}$ nonzero entries. By the union bound, we have:
\begin{align}
    \frac{\|\b{S}\|_0}
    {d_{\mathrm{in}}d_{\mathrm{out}}}
    &\leq
    \min\left\{1,r_sq_aq_b\right\}.
    \label{eq:roselora_density_bound}
\end{align}
Equivalently, by defining the zero sparsity of a matrix $\b{X}$ as:
\begin{align*}
    s(\b{X})
    &:=
    1-
    \frac{\|\b{X}\|_0}{\#\b{X}},
\end{align*}
where $\#\b{X}$ denotes the number of elements in matrix $\b{X}$, we obtain:
\begin{align}
    s(\b{S})
    &\geq
    \max\left\{0,1-r_sq_aq_b\right\}.
    \label{eq:roselora_sparsity_bound}
\end{align}

Unlike the sparse term in RoSA \cite{nikdan2024rosa} and classical RPCA \cite{candes2011robust}, the structured sparse component is also
low rank:
\begin{align}
    \operatorname{rank}(\b{S})
    &\leq r_s.
    \label{eq:sparse_component_rank}
\end{align}
Overall, the effective rank of LoRSA is upper-bounded by:
\begin{align}
\operatorname{rank}(\b{L}+\b{S})
    &\leq
    r_l+r_s.
\end{align}

\subsection{Task-Update Hypothesis and Approximation Advantage}

Let $\Delta \b{W}^{\star}$ denote an ideal task-specific parameter update.
We hypothesize that it admits the approximate decomposition:
\begin{align}
    \Delta \b{W}^{\star}
    &=
    \b{L}^{\star}+\b{S}^{\star}+\b{E}^{\star},
    \label{eq:ideal_update_decomposition}
\end{align}
where $\b{L}^{\star}$ contains the dominant, globally coordinated
adaptation, $\b{S}^{\star}$ contains localized residual corrections, and
$\b{E}^{\star}$ represents variation that is not captured by either
component.

Under this interpretation, the low-rank component is intended to capture the dominant low-dimensional structure of the adaptation. However, a fixed-rank LoRA update may not efficiently represent localized or lower-energy directions in the task-specific update. 
The structured sparse component enlarges the adaptation space in a controlled manner and can capture structured residual directions that remain after fitting the dense low-rank component.

To formalize this argument, we define the LoRA hypothesis class:
\begin{align}
    \mathcal{L}_{r_l}
    &:=
    \left\{
    \b{L}
    \in
    \mathbb{R}^{d_{\mathrm{out}}\times d_{\mathrm{in}}}
    :
    \operatorname{rank}(\b{L})\leq r_l
    \right\},
    \label{eq:lora_class}
\end{align}
and we define the structured sparse low-rank class:
\begin{equation}
\begin{aligned}
    &\mathcal{S}_{r_s,q_a,q_b}
    :=\\
    &\quad\left\{
    \b{B}_s\b{A}_s \Bigg|
    \begin{array}{l}
    \operatorname{rank}(\b{B}_s\b{A}_s)\leq r_s,\\
    \|(\b{A}_s)_{k:}\|_0
    \leq q_a d_{\mathrm{in}},\\
    \|(\b{B}_s)_{:k}\|_0
    \leq q_b d_{\mathrm{out}},
    \quad
    k=1,\ldots,r_s
    \end{array}
    \right\}.
    \label{eq:roselora_class}
\end{aligned}
\end{equation}
The combined hypothesis class is the Minkowski sum:
\begin{align}
    \mathcal{H}_{l+s}
    &=
    \mathcal{L}_{r_l}
    +
    \mathcal{S}_{r_s,q_a,q_b}
    \nonumber\\
    &=
    \left\{
    \b{L}+\b{S}:
    \b{L}\in\mathcal{L}_{r_l},
    \;
    \b{S}\in\mathcal{S}_{r_s,q_a,q_b}
    \right\}.
    \label{eq:combined_class}
\end{align}

\begin{proposition}[Hypothesis-class inclusion]
\label{prop:hypothesis_inclusion}
Assume that the zero matrix belongs to both $\mathcal{L}_{r_l}$ and
$\mathcal{S}_{r_s,q_a,q_b}$. Then:
\begin{equation}
\begin{aligned}
    &\mathcal{L}_{r_l}
    \subseteq
    \mathcal{H}_{l+s}, \\
    &\mathcal{S}_{r_s,q_a,q_b}
    \subseteq
    \mathcal{H}_{l+s}.
\end{aligned}
\end{equation}
\end{proposition}

\begin{proof}
For every $\b{L}\in\mathcal{L}_{r_l}$, selecting $\b{S}=\b{0}$ gives
$\b{L}=\b{L}+\b{0}\in\mathcal{H}_{l+s}$. Similarly, for every
$\b{S}\in\mathcal{S}_{r_s,q_a,q_b}$, selecting $\b{L}=\b{0}$ gives
$\b{S}=\b{0}+\b{S}\in\mathcal{H}_{l+s}$.
\end{proof}

The previous proposition shows that the combined adapter has at least the
representational capacity of either component used alone.

\begin{proposition}[Approximation advantage]
\label{prop:approximation_advantage}
For an arbitrary target update $\Delta \b{W}^{\star}$, define:
\begin{align}
    &\varepsilon_l
    :=
    \inf_{\b{L}\in\mathcal{L}_{r_l}}
    \|\Delta \b{W}^{\star}-\b{L}\|_F,
    \label{eq:lora_approximation_error}\\
    &\varepsilon_s
    :=
    \inf_{\b{S}\in\mathcal{S}_{r_s,q_a,q_b}}
    \|\Delta \b{W}^{\star}-\b{S}\|_F,
    \label{eq:rose_approximation_error}\\
    &\varepsilon_{l+s}
    :=
    \inf_{\substack{
    \b{L}\in\mathcal{L}_{r_l}\\
    \b{S}\in\mathcal{S}_{r_s,q_a,q_b}
    }}
    \|\Delta \b{W}^{\star}-\b{L}-\b{S}\|_F,
    \label{eq:combined_approximation_error}
\end{align}
where $\|.\|_F$ denotes the Frobenius norm.
Then:
\begin{align}
    \varepsilon_{l+s}
    &\leq
    \min\{\varepsilon_l,\varepsilon_s\}.
    \label{eq:approximation_error_bound}
\end{align}
\end{proposition}

\begin{proof}
The choice $\b{S}=\b{0}$ is feasible for the combined class. Therefore,
\begin{align*}
    \varepsilon_{l+s}
    &\leq
    \inf_{\b{L}\in\mathcal{L}_{r_l}}
    \|\Delta \b{W}^{\star}-\b{L}\|_F\\
    &=
    \varepsilon_l.
\end{align*}
Likewise, the choice $\b{L}=\b{0}$ is feasible, and hence:
\begin{align*}
    \varepsilon_{l+s}
    &\leq
    \inf_{\b{S}\in\mathcal{S}_{r_s,q_a,q_b}}
    \|\Delta \b{W}^{\star}-\b{S}\|_F\\
    &=
    \varepsilon_s.
\end{align*}
Combining the two inequalities proves
Eq.~\eqref{eq:approximation_error_bound}.
\end{proof}

The bound in Proposition~\ref{prop:approximation_advantage} concerns the best achievable approximation within each hypothesis class. It does not by itself guarantee that a nonconvex training procedure reaches the global optimum, nor does it guarantee lower test error. Nevertheless, it establishes that the combined parameterization can represent every solution available to LoRA alone or RoseLoRA alone, in addition to updates requiring both structures.

A more detailed interpretation can be obtained by considering the residual of the best low-rank approximation. Let:
\begin{align}
    \b{L}_{r_l}^{\star}
    &\in
    \arg\min_{\b{L}\in\mathcal{L}_{r_l}}
    \|\Delta \b{W}^{\star}-\b{L}\|_F^2,
    \label{eq:best_low_rank_update}
\end{align}
and define the remaining residual as:
\begin{align}
    \b{R}^{\star}
    &:=
    \Delta \b{W}^{\star}-\b{L}_{r_l}^{\star}.
    \label{eq:lora_residual}
\end{align}
The LoRA-only approximation error is:
\begin{align}
    \varepsilon_l^2
    &=
    \|\b{R}^{\star}\|_F^2.
    \label{eq:lora_residual_error}
\end{align}
Now let:
\begin{align}
    \b{S}_{\mathrm{Rose}}^{\star}
    &\in
    \arg\min_{\b{S}\in\mathcal{S}_{r_s,q_a,q_b}}
    \|\b{R}^{\star}-\b{S}\|_F^2.
    \label{eq:best_sparse_residual}
\end{align}
The corresponding combined approximation error is:
\begin{align}
    \tilde{\varepsilon}_{l+s}^2
    &:=
    \|\b{R}^{\star}-\b{S}_{\mathrm{Rose}}^{\star}\|_F^2,
    \label{eq:combined_residual_error}
\end{align}
where:
\begin{align}\label{equation_e_e_e}
\varepsilon_{l+s}^2 \leq \tilde{\varepsilon}_{l+s}^2 \leq \varepsilon_{l}^2.
\end{align}
The first inequality in Eq. (\ref{equation_e_e_e}) holds because the jointly optimized problem can select $\b{L} = \b{L}_{r_l}^{\star}$ and $\b{S} = \b{S}_{\mathrm{Rose}}^{\star}$. The second inequality in this equation holds because $\b{S}=\b{0}$ is feasible in the structured-sparse class. 

Expanding the difference between the LoRA-only error and the sequential combined error gives:
\begin{align}
    \varepsilon_l^2-\tilde{\varepsilon}_{l+s}^2
    &=
    \|\b{R}^{\star}\|_F^2
    -
    \|\b{R}^{\star}-\b{S}_{\mathrm{Rose}}^{\star}\|_F^2
    \nonumber\\
    &=
    2
    \left\langle
    \b{R}^{\star},
    \b{S}_{\mathrm{Rose}}^{\star}
    \right\rangle_F
    -
    \left\|
    \b{S}_{\mathrm{Rose}}^{\star}
    \right\|_F^2,
    \label{eq:improvement_identity}
\end{align}
where $\langle \cdot, \cdot \rangle_F$ denotes the inner product of the matrices under the Frobenius norm. 
Therefore, the sequential structured-sparse correction strictly improves upon the best low-rank-only approximation whenever:
\begin{align}
    2
    \left\langle
    \b{R}^{\star},
    \b{S}_{\mathrm{Rose}}^{\star}
    \right\rangle_F
    &>
    \left\|
    \b{S}_{\mathrm{Rose}}^{\star}
    \right\|_F^2.
    \label{eq:strict_improvement_condition}
\end{align}
This condition states that the learned structured sparse update must be
sufficiently aligned with the residual left by the dense low-rank update.

\subsection{Singular-Subspace Expansion and Component Complementarity}
\label{sec:singular_subspace_analysis}

The decomposition of LoRSA provides an interpretation in terms of the singular subspaces accessible to the task-specific update. Let the compact singular value decompositions of the two adaptation components be:
\begin{align}
    \b{L}
    &=
    \b{U}_l
    \b{\Sigma}_l
    \b{V}_l^{\top},
    \label{eq:svd_l}\\
    \b{S}
    &=
    \b{U}_s
    \b{\Sigma}_s
    \b{V}_s^{\top},
    \label{eq:svd_s}
\end{align}
where the columns of $\b{U}_l$ and $\b{U}_s$ span the output-side
singular subspaces, while the columns of $\b{V}_l$ and $\b{V}_s$ span
the input-side singular subspaces.

For the LoRSA update:
\begin{align*}
    \Delta\b{W}
    &=
    \b{L}+\b{S},
\end{align*}
its column and row spaces satisfy:
\begin{align}
    \operatorname{col}(\Delta\b{W})
    &\subseteq
    \operatorname{col}(\b{L})
    +
    \operatorname{col}(\b{S}),
    \label{eq:column_space_inclusion}\\
    \operatorname{row}(\Delta\b{W})
    &\subseteq
    \operatorname{row}(\b{L})
    +
    \operatorname{row}(\b{S}).
    \label{eq:row_space_inclusion}
\end{align}
Consequently:
\begin{align}
    \operatorname{rank}(\Delta\b{W})
    &\leq
    \operatorname{rank}(\b{L})
    +
    \operatorname{rank}(\b{S})
    \nonumber\\
    &\leq
    r_l+r_s.
    \label{eq:LoRSA_rank_upper_bound}
\end{align}

Equation~\eqref{eq:LoRSA_rank_upper_bound} shows that LoRSA may access a
larger adaptation subspace than LoRA or RoseLoRA alone. However, the
increase is not automatic. If the two components have strongly overlapping
singular subspaces, or if they partially cancel each other, the rank of
their sum may be considerably smaller than $r_l+r_s$.

The dimension of the combined output-side subspace is:
\begin{equation}
\begin{aligned}
    &\dim\left(
    \operatorname{col}(\b{L})
    +
    \operatorname{col}(\b{S})
    \right)
    \\
    &\qquad=
    \operatorname{rank}(\b{L})
    +
    \operatorname{rank}(\b{S})
    -
    \dim\left(
    \operatorname{col}(\b{L})
    \cap
    \operatorname{col}(\b{S})
    \right).
    \label{eq:subspace_sum_dimension}
\end{aligned}
\end{equation}
An analogous identity holds for the row spaces. Therefore, LoRSA enlarges the accessible adaptation space when the structured-sparse component introduces directions that are not already contained in the dense low-rank component.

\begin{proposition}[Rank additivity under subspace orthogonality]
\label{prop:rank_additivity}
Suppose that the row and column singular subspaces of $\b{L}$ and $\b{S}$
are mutually orthogonal:
\begin{align}
    &\b{U}_l^{\top}\b{U}_s
    =
    \b{0},
    \quad
    \b{V}_l^{\top}\b{V}_s
    =
    \b{0}.
    \label{eq:two_sided_orthogonality}
\end{align}
Then:
\begin{align}
    \operatorname{rank}(\b{L}+\b{S})
    &=
    \operatorname{rank}(\b{L})
    +
    \operatorname{rank}(\b{S}).
    \label{eq:orthogonal_rank_sum}
\end{align}
Moreover, the nonzero singular values of $\b{L}+\b{S}$ are the union of
the nonzero singular values of $\b{L}$ and $\b{S}$.
\end{proposition}

\begin{proof}
Under Eq.~\eqref{eq:two_sided_orthogonality}, the concatenated matrices $[\b{U}_l, \b{U}_s]$ and $[\b{V}_l, \b{V}_s]$ have orthonormal columns. Therefore:
\begin{align*}
    \b{L}+\b{S}
    &=
    \begin{bmatrix}
        \b{U}_l & \b{U}_s
    \end{bmatrix}
    \begin{bmatrix}
        \b{\Sigma}_l & \b{0}\\
        \b{0} & \b{\Sigma}_s
    \end{bmatrix}
    \begin{bmatrix}
        \b{V}_l & \b{V}_s
    \end{bmatrix}^{\top}.
\end{align*}
The middle matrix is block diagonal. Its nonzero singular values are
therefore the union of those of $\b{\Sigma}_l$ and $\b{\Sigma}_s$, and
its rank is the sum of their ranks.
\end{proof}

This proposition describes an idealized case of complete complementarity.
In practice, the two components need not be exactly orthogonal. 
Their degree of overlap can be quantified as explained in the following. 
The part of $\b{L}$ lying in both the row and column singular spaces of $\b{S}$ is $\b{U}_s \b{U}_s^\top \b{L} \b{V}_s \b{V}_s^\top$. The part of $\b{L}$ not represented by the singular directions of $\b{S}$ is $\b{L} - \b{U}_s \b{U}_s^\top \b{L} \b{V}_s \b{V}_s^\top$. According to orthogonal projection in the Frobenius inner-product space, we define:
\begin{align}\label{equation_nu_L_S}
&\nu_{L|S} = 1 - \frac{\|\b{U}_s \b{U}_s^\top \b{L} \b{V}_s \b{V}_s^\top\|_F^2}{\|\b{L}\|_F^2},
\end{align}
as the measurement of the part of $\b{L}$ not shared with $\b{S}$. 
Likewise, the measurement of the part of $\b{S}$ not shared with $\b{L}$ can be formulated as:
\begin{align}\label{equation_nu_S_L}
&\nu_{S|L} = 1 - \frac{\|\b{U}_l \b{U}_l^\top \b{S} \b{V}_l \b{V}_l^\top\|_F^2}{\|\b{S}\|_F^2}.
\end{align}
High values for Eqs. (\ref{equation_nu_L_S}) and (\ref{equation_nu_S_L}) indicate that the two components of LoRSA have captured distinct valuable information.

\subsection{Theory of Domain Generalization by LoRSA}
\label{sec:domain_generalization_theory}

Domain generalization aims to learn from one or more observed source domains while generalizing to target domains that are unavailable during training \cite{muandet2013domain,wang2022generalizing}. Therefore, LoRSA cannot directly optimize a discrepancy involving a particular unseen target distribution. Its potential benefit instead follows from the inductive bias imposed on the task-specific parameter update.

Let $\Delta\b{W}^{\star}_d$ denote an ideal task-specific update for domain $d$. Motivated by invariant-representation approaches to domain generalization~\cite{muandet2013domain}, we make the following structural assumption:
\begin{align}
    \Delta\b{W}^{\star}_d
    &=
    \underbrace{
    \b{L}^{\star}+\b{S}^{\star}
    }_{\text{domain-stable task component}}
    +
    \underbrace{
    \b{N}_d
    }_{\text{domain-specific nuisance}}.
    \label{eq:domain_update_decomposition}
\end{align}
Here, $\b{L}^{\star}$ represents dominant low-rank structure shared across domains, $\b{S}^{\star}$ represents structured task-relevant residual directions that are also stable across domains, and $\b{N}_d$ contains domain-dependent variation. 

On the one hand, to connect parameter approximation to the task objective, let $\mathcal{J}_d(\Delta\b{W})$ denote the population loss in domain $d$. If $\mathcal{J}_d$ is twice differentiable and $\Delta\b{W}^{\star}_d$ is a local minimizer, a second-order Taylor expansion gives:
\begin{align}
    \mathcal{J}_d(&\Delta\b{W})
    =
    \mathcal{J}_d(\Delta\b{W}^{\star}_d) \nonumber
    \\
    &+
    \frac{1}{2}
    \operatorname{vec}
    \left(
    \Delta\b{W}-\Delta\b{W}^{\star}_d
    \right)^{\top}
    \b{H}_d
    \operatorname{vec}
    \left(
    \Delta\b{W}-\Delta\b{W}^{\star}_d
    \right)
    \nonumber\\
    &+
    o\left(
    \left\|
    \Delta\b{W}-\Delta\b{W}^{\star}_d
    \right\|_F^2
    \right),
    \label{eq:local_quadratic_loss}
\end{align}
where $\b{H}_d$ is the Hessian evaluated at $\Delta\b{W}^{\star}_d$ and $o(.)$ denotes the small-O notation in complexity. 
If the Hessian eigenvalues satisfy:
\begin{align}
    \mu_d\b{I}
    \preceq
    \b{H}_d
    \preceq
    \beta_d\b{I},
    \label{eq:hessian_bounds}
\end{align}
for some $0<\mu_d\leq\beta_d$, then Eq.~\eqref{eq:local_quadratic_loss} implies that, locally, there is:
\begin{align}
    \frac{\mu_d}{2}
    \left\|
    \Delta\b{W}-\Delta\b{W}^{\star}_d
    \right\|_F^2
    &\lesssim
    \mathcal{J}_d(\Delta\b{W})
    -
    \mathcal{J}_d(\Delta\b{W}^{\star}_d)
    \nonumber\\
    &\lesssim
    \frac{\beta_d}{2}
    \left\|
    \Delta\b{W}-\Delta\b{W}^{\star}_d
    \right\|_F^2.
    \label{eq:loss_parameter_error_relation}
\end{align}
Thus, within this local regime, reducing parameter-space approximation error also reduces an upper bound on the excess domain loss.

On the other hand, LoRA searches only within the low-rank class $\mathcal{L}_{r_l}$, whereas LoRSA searches within:
\begin{align}
    \mathcal{H}_{l+s}
    &=
    \mathcal{L}_{r_l}
    +
    \mathcal{S}_{r_s,q_a,q_b},
    \label{eq:dg_lorsa_hypothesis_class}
\end{align}
where $\mathcal{L}_{r_l}$ and $\mathcal{S}_{r_s,q_a,q_b}$ are defined in Eqs. (\ref{eq:lora_class}) and (\ref{eq:roselora_class}), respectively. 
The class in Eq. \eqref{eq:dg_lorsa_hypothesis_class} is the Minkowski sum introduced earlier in Eq. (\ref{eq:combined_class}). If $\b{L}^{\star}\in\mathcal{L}_{r_l}$, $\b{S}^{\star}\in\mathcal{S}_{r_s,q_a,q_b}$, and $\b{S}^{\star}\notin\mathcal{L}_{r_l}$, then a fixed-rank LoRA update cannot generally represent the complete domain-stable component, whereas LoRSA can represent both terms.

\begin{proposition}[Target-domain approximation advantage]
\label{prop:target_domain_approximation}
Consider an unseen target domain $t$ satisfying
Eq. \eqref{eq:domain_update_decomposition}. Assume that:
\begin{align*}
    \b{L}^{\star}
    \in
    \mathcal{L}_{r_l},
    \quad
    \b{S}^{\star}
    \in
    \mathcal{S}_{r_s,q_a,q_b},
\end{align*}
and consider the idealized learned updates:
\begin{align*}
    \Delta\b{W}_{\mathrm{LoRA}}
    =
    \b{L}^{\star},
    \quad
    \Delta\b{W}_{\mathrm{LoRSA}}
    =
    \b{L}^{\star}+\b{S}^{\star}.
\end{align*}
Then, their target-domain approximation errors satisfy:
\begin{align}
    &\varepsilon_{\mathrm{LoRA},t}^{2}
    :=
    \left\|
    \Delta\b{W}^{\star}_t
    -
    \Delta\b{W}_{\mathrm{LoRA}}
    \right\|_F^2
    =
    \left\|
    \b{S}^{\star}+\b{N}_t
    \right\|_F^2,
    \label{eq:lora_target_error}\\
    &\varepsilon_{\mathrm{LoRSA},t}^{2}
    :=
    \left\|
    \Delta\b{W}^{\star}_t
    -
    \Delta\b{W}_{\mathrm{LoRSA}}
    \right\|_F^2
    =
    \left\|
    \b{N}_t
    \right\|_F^2.
    \label{eq:lorsa_target_error}
\end{align}
Moreover, LoRSA has lower target-domain approximation error whenever the following condition is satisfied:
\begin{align}
    \left\langle
    \b{S}^{\star},
    \b{N}_t
    \right\rangle_F
    &>
    -
    \frac{1}{2}
    \left\|
    \b{S}^{\star}
    \right\|_F^2.
    \label{eq:domain_generalization_condition}
\end{align}
\end{proposition}

\begin{proof}
Substituting Eq.~\eqref{eq:domain_update_decomposition} into the two error definitions gives Eqs.~\eqref{eq:lora_target_error} and \eqref{eq:lorsa_target_error}. Using the Frobenius inner-product identity \cite{golub2013matrix}:
\begin{align*}
    \|\b{A}+\b{B}\|_F^2
    &=
    \|\b{A}\|_F^2
    +
    \|\b{B}\|_F^2
    +
    2\langle\b{A},\b{B}\rangle_F,
\end{align*}
we obtain:
\begin{align}
    \varepsilon_{\mathrm{LoRA},t}^{2}
    -
    \varepsilon_{\mathrm{LoRSA},t}^{2}
    &=
    \left\|
    \b{S}^{\star}
    \right\|_F^2
    +
    2
    \left\langle
    \b{S}^{\star},
    \b{N}_t
    \right\rangle_F.
    \label{eq:domain_generalization_improvement}
\end{align}
The right-hand side is positive exactly when
Eq.~\eqref{eq:domain_generalization_condition} holds.
\end{proof}

A particularly interpretable special case occurs when the stable structured component is orthogonal to the target-domain nuisance, i.e., the stable residual structure does not align with target-specific nuisance directions:
\begin{align}
    \left\langle
    \b{S}^{\star},
    \b{N}_t
    \right\rangle_F
    &=
    \b{0}.
    \label{eq:stable_nuisance_orthogonality}
\end{align}
Under this assumption, Eq. (\ref{eq:domain_generalization_improvement}) in Proposition \ref{prop:target_domain_approximation} reduces to:
\begin{align}
    \varepsilon_{\mathrm{LoRA},t}^{2}
    -
    \varepsilon_{\mathrm{LoRSA},t}^{2}
    &=
    \left\|
    \b{S}^{\star}
    \right\|_F^2
    >
    0.
    \label{eq:orthogonal_domain_improvement}
\end{align}

Combining Proposition~\ref{prop:target_domain_approximation} with the local quadratic relation in Eq.~\eqref{eq:loss_parameter_error_relation} shows that, under the stated smoothness and Hessian assumptions, a lower target-domain parameter approximation error translates into a lower local upper bound on the target-domain excess loss. Hence, LoRSA can improve domain generalization when the structured-sparse component captures domain-stable residual directions omitted by the selected LoRA rank.

\subsection{Distinction from LoRA, RoseLoRA, and RoSA}
\label{sec:comparison_with_existing_adapters}

LoRSA is related to LoRA \cite{hu2022lora}, RoseLoRA \cite{wang2024roselora}, and RoSA \cite{nikdan2024rosa}, but differs from each method
in its parameterization and in the structural assumptions imposed on the
task-specific update. To make these differences explicit, we state the
adapted weight matrix of each method using a common notation.

\subsubsection{Comparison with LoRA}

LoRA \cite{hu2022lora} models the task-specific update using a single dense low-rank
component:
\begin{align}
    \b{W}_{\mathrm{LoRA}}
    &=
    \b{W}_0
    +
    \frac{\alpha_l}{r_l}
    \b{B}_l\b{A}_l,
    \label{eq:lora_comparison}
\end{align}
where
$\b{A}_l\in\mathbb{R}^{r_l\times d_{\mathrm{in}}}$ and
$\b{B}_l\in\mathbb{R}^{d_{\mathrm{out}}\times r_l}$. Consequently:
\begin{align*}
    \operatorname{rank}
    \left(
    \b{W}_{\mathrm{LoRA}}-\b{W}_0
    \right)
    &\leq r_l.
\end{align*}
Although the update is low rank, it is generally dense. LoRA therefore
assumes that the task-specific adaptation can be represented sufficiently
well by a single globally coordinated low-dimensional component.

In contrast, LoRSA augments the dense low-rank component with a second,
structured-sparse component:
\begin{align}
    \b{W}_{\mathrm{LoRSA}}
    &=
    \b{W}_0
    +
    \frac{\alpha_l}{r_l}
    \b{B}_l\b{A}_l
    +
    \frac{\alpha_s}{r_s}
    \b{B}_s\b{A}_s.
    \label{eq:LoRSA_comparison}
\end{align}
Therefore, LoRA is recovered as a special case of LoRSA by setting
$\b{B}_s\b{A}_s=\b{0}$. The additional component allows LoRSA to represent structured residual corrections that may not be represented efficiently by the selected LoRA rank.

\subsubsection{Comparison with RoseLoRA}

RoseLoRA \cite{wang2024roselora} uses a single structured-sparse low-rank update:
\begin{align}
    \b{W}_{\mathrm{RoseLoRA}}
    &=
    \b{W}_0
    +
    \frac{\alpha_s}{r_s}
    \b{B}_s\b{A}_s,
    \label{eq:roselora_comparison}
\end{align}
where row-wise sparsity is imposed on $\b{A}_s$ and column-wise sparsity
is imposed on $\b{B}_s$. In particular:
\begin{align}
    &\left\|(\b{A}_s)_{k:}\right\|_0
    \leq
    q_a d_{\mathrm{in}},
    \quad
    \left\|(\b{B}_s)_{:k}\right\|_0
    \leq
    q_b d_{\mathrm{out}},
    \quad
    k=1,\ldots,r_s.
    \label{eq:roselora_comparison_constraints}
\end{align}
Its effective update is simultaneously low rank and structured sparse:
\begin{align*}
    \operatorname{rank}(\b{B}_s\b{A}_s)
    &\leq r_s,\\
    \operatorname{supp}(\b{B}_s\b{A}_s)
    &\subseteq
    \bigcup_{k=1}^{r_s}
    \mathcal{I}_k\times\mathcal{J}_k,
\end{align*}
where
$\mathcal{I}_k=\operatorname{supp}((\b{B}_s)_{:k})$ and
$\mathcal{J}_k=\operatorname{supp}((\b{A}_s)_{k:})$.

RoseLoRA alone does not contain a separate dense low-rank component.
Consequently, all task-specific changes must be represented through the
structured-sparse factors. LoRSA retains the structured sparse component but combines it with a conventional dense low-rank component. RoseLoRA is therefore recovered as a special case of LoRSA by setting $\b{B}_l\b{A}_l=\b{0}$. The proposed LoRSA method separates the adaptation into two complementary components:
\begin{align*}
    \underbrace{\b{L}}_{\text{dense low-rank component}}
    \qquad\text{and}\qquad
    \underbrace{\b{S}}_{\text{structured-sparse low-rank component}}.
\end{align*}

\subsubsection{Comparison with RoSA}

RoSA \cite{nikdan2024rosa} is also motivated by the low-rank-plus-sparse decomposition of RPCA and parameterizes the adapted weight matrix as:
\begin{align}
    \b{W}_{\mathrm{RoSA}}
    &=
    \b{W}_0
    +
    \b{L}_{\mathrm{RoSA}}
    +
    \b{S}_{\Omega},
    \label{eq:rosa_comparison}
\end{align}
where $\b{L}_{\mathrm{RoSA}}$ is a low-rank matrix and
$\b{S}_{\Omega}$ is a directly parameterized sparse matrix whose nonzero
entries are restricted to a support set $\Omega$. A representative
parameterization is:
\begin{align*}
    &\b{L}_{\mathrm{RoSA}}
    =
    \b{B}\b{A}, \\
    &\operatorname{supp}(\b{S}_{\Omega})
    \subseteq \Omega.
\end{align*}
The sparse component $\b{S}_{\Omega}$ is not required to admit a low-rank
factorization. It may therefore have rank as large as
$\min\{d_{\mathrm{in}},d_{\mathrm{out}}\}$, despite containing relatively
few nonzero entries.

LoRSA differs from RoSA in the representation of the sparse component.
Instead of directly learning an unrestricted sparse matrix, LoRSA defines:
\begin{align}
    \b{S}_{\mathrm{LoRSA}}
    &=
    \frac{\alpha_s}{r_s}
    \b{B}_s\b{A}_s,
    \label{eq:LoRSA_sparse_component_comparison}
\end{align}
where $\b{A}_s$ and $\b{B}_s$ satisfy row- and column-wise sparsity
constraints. Hence:
\begin{align}
    \operatorname{rank}(\b{S}_{\mathrm{LoRSA}})
    &\leq r_s,
    \label{eq:LoRSA_sparse_rank_comparison}\\
    \operatorname{supp}(\b{S}_{\mathrm{LoRSA}})
    &\subseteq
    \bigcup_{k=1}^{r_s}
    \mathcal{I}_k\times\mathcal{J}_k.
    \label{eq:LoRSA_sparse_support_comparison}
\end{align}
The sparse component of LoRSA is thus constrained in two ways: it has a bounded rank and its support is induced by unions of row-column Cartesian products.
By contrast, the sparse component of RoSA is directly parameterized on an arbitrary sparse support and is not constrained to be low rank.

This distinction changes the inductive bias of the two methods. RoSA seeks to approximate a full-fine-tuning update using the classical RPCA-inspired combination of a low-rank matrix and an unrestricted highly sparse matrix.
LoRSA instead combines a dense low-rank update with a structured-sparse
low-rank correction. Thus, LoRSA is more restrictive than RoSA in the class of sparse matrices it can represent, but it preserves a factorized low-rank parameterization for both adaptation components.
Moreover, in contrast to RoSA, the locations of the nonzero entries in the structured sparse component of LoRSA are not fixed during training.

The four methods can be summarized as:
\begin{align}
    \Delta\b{W}_{\mathrm{LoRA}}
    &=
    \b{L},
    \label{eq:summary_lora}\\
    \Delta\b{W}_{\mathrm{RoseLoRA}}
    &=
    \b{S}_{\mathrm{slr}},
    \label{eq:summary_roselora}\\
    \Delta\b{W}_{\mathrm{RoSA}}
    &=
    \b{L}+\b{S}_{\mathrm{direct}},
    \label{eq:summary_rosa}\\
    \Delta\b{W}_{\mathrm{LoRSA}}
    &=
    \b{L}+\b{S}_{\mathrm{slr}},
    \label{eq:summary_LoRSA}
\end{align}
where $\b{S}_{\mathrm{direct}}$ denotes a directly parameterized sparse
matrix and $\b{S}_{\mathrm{slr}}$ denotes a structured-sparse low-rank
matrix. The defining distinction of LoRSA is therefore not merely the
addition of low-rank and sparse updates, but the combination of a dense
low-rank component with a RoseLoRA-based structured-sparse low-rank
component.

\section{Experimental Results}

\subsection{Datasets and the Foundation Model}
We utilized three distinct mammography datasets in this study: VinDr-Mammo \cite{nguyen2023vindr} (Vietnam), RSNA \cite{rsna_smbc_2024} (USA and Australia), and MammosighTR \cite{kocc2025mammosightr} (Turkey). These cohorts were specifically selected because they offer compatible four-category breast density annotations while encompassing varied geographic populations, healthcare systems, and imaging environments. Consequently, this selection facilitated a robust evaluation of both \textit{in-domain performance} and \textit{cross-domain generalization}.

VinDr-Mammo was used as the source-domain dataset. It contains 5,000 four-view mammography examinations collected from two hospitals in Vietnam, with an official split of 4,000 training and 1,000 test examinations. The training partition was further divided at the patient level into training and validation subsets for model optimization and selection, while the official test set was kept fully held out for internal evaluation.

RSNA and MammosighTR were treated as unseen target domains and were not used during training, validation, or hyperparameter selection. Although the RSNA challenge included hidden public and private test cohorts, the publicly distributed test folder contains only a small example set. Therefore, the entire labeled RSNA training release was used exclusively for external evaluation. MammosighTR provides predefined training and test partitions, and only its official test set was used for external testing.

We used DINOv3-Base, corresponding to the ViT-B/16 architecture, as the common vision foundation model backbone for all experiments. DINOv3 is a self-supervised model pretrained on the large-scale LVD-1689M image corpus to learn transferable global and dense visual representations \cite{simeoni2025dinov3}. The Base variant provides a strong yet computationally practical backbone for evaluating parameter-efficient adaptation.

\subsection{Internal Performance and External Generalization}
\label{sec:internal_external_generalization}

For the experiments, we used the DINOv3-Base vision foundation model \cite{simeoni2025dinov3} as a frozen baseline and adapted it using LoRA, RoSE-LoRA, RoSA, and our proposed method. For the single-component methods, namely LoRA and RoSE-LoRA, in addition to rank $8$, we also evaluated a rank of $16$ to ensure a fair comparison with the dual-component approaches.

Tables~\ref{tab:vindr_internal_validation}--\ref{tab:rsna_external_dino_base}
reveal a clear distinction between source-domain performance and
generalization to unseen cohorts. On the VinDr-Mammo internal evaluation
set, several baselines remain competitive with the proposed method.
RoseLoRA with rank $16$ obtains the highest accuracy, weighted F1, and
macro-AUROC, while RoSA achieves the highest macro-F1. The proposed
LoRA--RoseLoRA adapter instead obtains the highest balanced accuracy and
sensitivity and remains close to the best-performing baseline on the other
metrics. Thus, the internal results do not indicate uniform dominance by
one adaptation method.

The difference becomes clearer on the external cohorts. On MammosighTR,
the proposed method achieves the best result for every reported metric,
including an accuracy of $0.6656$, balanced accuracy of $0.5011$,
macro-F1 of $0.5150$, and QWK of $0.6855$. Similarly, on RSNA, it
achieves the highest accuracy, balanced accuracy, macro-F1, weighted F1,
macro-AUROC, and QWK. The external improvements are particularly evident
in the class-level metrics. Relative to the strongest competing result,
the proposed method improves macro-F1 from $0.4935$ to $0.5150$ on
MammosighTR and from $0.5536$ to $0.5845$ on RSNA. Therefore, although
the baseline methods can match or exceed the proposed method on individual
internal metrics, the proposed parameterization provides the most
consistent performance after transfer to the two unseen cohorts. This
pattern supports the intended role of LoRSA as a generalization-oriented
adapter rather than simply an adapter optimized for the source validation
set.

The behavior of RoSA further illustrates the importance of how the sparse
adaptation component is constructed. RoSA performs strongly on RSNA,
where it is the second-best method for most metrics and achieves a
specificity of $0.8881$, marginally higher than that of the proposed
method. However, its advantage is less consistent on MammosighTR, where
rank-$16$ LoRA outperforms RoSA on all reported metrics. One plausible
explanation is the use of a static sparse mask in RoSA, which is generated
from a small calibration subset before the main adapter training. Such a
mask may identify sparse coordinates that are useful for the source domain
and for some external shifts, while excluding other low-amplitude or subtle
task-relevant directions needed under a different external distribution.
This interpretation is consistent with RoSA's strong performance on RSNA
but weaker relative performance on MammosighTR. Nevertheless, the current
results do not directly identify which representation directions are
omitted by the mask, and this mechanism should therefore be viewed as a
plausible explanation rather than a causal conclusion.

The rank ablation also shows that the effect of increasing adapter capacity
depends on the structure of the adapter. Increasing the LoRA rank from
$8$ to $16$ substantially improves its performance on both external
cohorts. On MammosighTR, balanced accuracy increases from $0.4309$ to
$0.4883$, while macro-F1 increases from $0.3974$ to $0.4935$. On
RSNA, the corresponding metrics increase from $0.4402$ to $0.5070$
and from $0.4340$ to $0.5211$, respectively. These results suggest
that rank-$8$ LoRA is capacity-limited for this task and that increasing
its rank allows it to capture additional task-relevant directions without
producing a corresponding loss of external generalization.

The opposite trend is observed for RoseLoRA. Increasing its rank from
$8$ to $16$ improves several internal metrics, including accuracy and
macro-AUROC, but degrades nearly all external metrics on both MammosighTR
and RSNA. For example, its MammosighTR macro-F1 decreases from $0.4426$
to $0.3868$, and its RSNA accuracy decreases from $0.7024$ to
$0.6525$. This divergence between internal and external performance
suggests that the higher-rank structured-sparse adapter may overestimate
the complexity required for the sparse component. Increasing the rank
increases the number of sparse rank-one components available to the
adapter and therefore enlarges the set of structured corrections that it
can represent. While this additional capacity can improve fitting of the
source-domain data, it may also allow the sparse component to capture
source-specific residual structure that does not transfer to unseen
cohorts.

The proposed method provides a more favorable allocation of the available
capacity. Its two rank-$8$ components have the same total nominal rank as
a rank-$16$ single-component adapter, yet it consistently outperforms both
rank-$16$ LoRA and rank-$16$ RoseLoRA on the external evaluations.
This comparison indicates that the external improvements cannot be
attributed only to a larger nominal rank. Rather, the results suggest that
separating the adaptation into a dense low-rank component and a
structured-sparse low-rank component provides a more effective balance:
the dense component captures broadly shared task structure, while the
structured component contributes additional corrections without requiring
the entire adaptation to follow a single dense or sparse parameterization.

Overall, the experiments show that strong internal performance does not
necessarily translate into reliable external performance. The proposed
method is not uniformly superior on the internal evaluation, but it is
the only method that achieves the best performance consistently
across both external cohorts. Together with the contrasting rank behavior
of LoRA and RoseLoRA and the cohort-dependent performance of RoSA, these
results support the central hypothesis that both the capacity and the
structural organization of the adaptation update influence
out-of-domain generalization.

\subsection{Weight-Matrix Complementarity Analysis}
\label{sec:weight_matrix_complementarity}

To examine whether the dense low-rank and structured-sparse
components learn complementary weight updates, we evaluated the
non-shared-energy measures in Eqs.~(\ref{equation_nu_L_S}) and~(\ref{equation_nu_S_L}) for the adapted
weight matrices. For every matrix, we reconstructed the effective
updates \(L\) and \(S\), computed their compact singular value
decompositions, and measured the fractions of their Frobenius
energies that could not be represented within the bilateral
row-and-column singular subspace of the other component.

The measurements are as follows:
\begin{align}
&\nu_{L|S} = 0.923407 \pm 0.084794, \\
&\nu_{S|L} = 0.926109 \pm 0.076872,
\end{align}
where the reported values are the mean and standard deviation
across the adapted weight matrices. Thus, on average,
approximately \(92.34\%\) of the energy of the dense low-rank
component lies outside the bilateral singular subspace of the
structured-sparse component. Conversely, approximately \(92.61\%\)
of the energy of the structured-sparse component lies outside the
corresponding singular subspace of the dense component.

Equivalently, the fractions of energy represented within the shared
bilateral singular subspaces are:
\begin{align}
&1-\nu_{L|S} = 0.076593, \\
&1-\nu_{S|L} = 0.073891,
\end{align}
for \(L\) and \(S\), respectively. 
Therefore, only approximately \(7.7\%\) of the energy
of either component is represented through the row and column
singular directions of the other component. The similar values of
the two directional measurements also indicate that the result is
not caused by one component being largely contained in the
subspace of the other. Instead, both components contain substantial
energy outside the singular geometry learned by their counterpart.
Hence, the two
components act as complementary weight-update paths rather than
redundant parameterizations of the same update.

\begin{table*}[t]
\centering
\caption{Internal evaluation of parameter-efficient adaptation methods
for four-class breast density classification on the VinDr-Mammo dataset.\\}
\label{tab:vindr_internal_validation}

\setlength{\tabcolsep}{3.2pt}
\renewcommand{\arraystretch}{1.14}
\scriptsize

\begin{threeparttable}
\begin{tabular}{
    @{}
    l
    c
    c c c c c c c c
    @{}
}
\toprule
\multirow{2}{*}{Method}
&
\multirow{2}{*}{$r$}
&
\multicolumn{8}{c}{Internal Validation Performance}
\\
\cmidrule(lr){3-10}

&
&
\makecell{Acc.}
&
\makecell{Bal. Acc.}
&
\makecell{Macro-F1}
&
\makecell{Weighted-F1}
&
\makecell{Macro-AUROC}
&
\makecell{QWK}
&
\makecell{Sens.}
&
\makecell{Spec.}
\\
\midrule

DINO (frozen)
& --
& 0.8105
& 0.5573
& 0.5283
& 0.8126
& 0.8994
& 0.6310
& 0.5573
& 0.8870
\\

\addlinespace[1.5pt]
DINO-Base + LoRA
& 8
& 0.8460
& 0.5497
& 0.5844
& 0.8368
& 0.9317
& 0.6478
& 0.5497
& 0.8784
\\

DINO-Base + LoRA
& 16
& 0.8145
& \underline{0.6415}
& 0.6020
& 0.8243
& \underline{0.9453}
& \underline{0.6809}
& \underline{0.6415}
& \textbf{0.9137}
\\

\addlinespace[1.5pt]
DINO-Base + RoseLoRA
& 8
& 0.8348
& 0.5950
& 0.5825
& 0.8379
& 0.9387
& \textbf{0.6811}
& 0.5950
& 0.9048
\\

DINO-Base + RoseLoRA
& 16
& \textbf{0.8520}
& 0.5692
& 0.6074
& \textbf{0.8444}
& \textbf{0.9471}
& 0.6653
& 0.5692
& 0.8838
\\

\addlinespace[1.5pt]
DINO-Base + RoSA
& 8
& \underline{0.8485}
& 0.5899
& \textbf{0.6233}
& \underline{0.8425}
& 0.9428
& 0.6674
& 0.5899
& 0.8861
\\

\midrule
\textbf{DINO-Base + LoRSA (Ours)}
& \textbf{8}
& 0.8168
& \textbf{0.6523}
& \underline{0.6201}
& 0.8257
& 0.9433
& 0.6785
& \textbf{0.6523}
& \underline{0.9096}
\\

\bottomrule
\end{tabular}

\begin{tablenotes}[flushleft]
\footnotesize
\item Acc.: accuracy; Bal. Acc.: balanced accuracy; QWK: quadratic
weighted Cohen's kappa; Sens.: sensitivity; Spec.: specificity.
\item Macro-AUROC denotes the macro-averaged one-vs-rest area under the
receiver operating characteristic curve.
\item Sensitivity and specificity were macro-averaged across the four
density categories using a one-vs-rest formulation.
\item QWK was calculated by treating breast density categories A--D as
ordered classes.
\item Boldface and underlining indicate the best and second-best results,
respectively.
\end{tablenotes}
\end{threeparttable}
\end{table*}

\begin{table*}[t]
\centering
\caption{External evaluation of parameter-efficient adaptation methods
using the DINO-Base backbone for four-class breast density classification
on the MammosighTR test set.\\}
\label{tab:mammosight_external_dino_base}

\setlength{\tabcolsep}{3.2pt}
\renewcommand{\arraystretch}{1.14}
\scriptsize

\begin{threeparttable}
\begin{tabular}{
    @{}
    l
    c
    c c c c c c c c
    @{}
}
\toprule
\multirow{2}{*}{Method}
&
\multirow{2}{*}{$r$}
&
\multicolumn{8}{c}{External Test Performance}
\\
\cmidrule(lr){3-10}

&
&
\makecell{Acc.}
&
\makecell{Bal. Acc.}
&
\makecell{Macro-F1}
&
\makecell{Weighted-F1}
&
\makecell{Macro-AUROC}
&
\makecell{QWK}
&
\makecell{Sens.}
&
\makecell{Spec.}
\\
\midrule

DINO-Base (frozen)
& --
& 0.6235
& 0.4390
& 0.4187
& 0.5655
& 0.8320
& 0.6215
& 0.4390
& 0.8489
\\

\addlinespace[1.5pt]
DINO-Base + LoRA
& 8
& 0.6288
& 0.4309
& 0.3974
& 0.5595
& 0.8825
& 0.6221
& 0.4309
& 0.8513
\\

DINO-Base + LoRA
& 16
& \underline{0.6576}
& \underline{0.4883}
& \underline{0.4935}
& \underline{0.6142}
& \underline{0.9029}
& \underline{0.6742}
& \underline{0.4883}
& \underline{0.8633}
\\

\addlinespace[1.5pt]
DINO-Base + RoseLoRA
& 8
& 0.6485
& 0.4577
& 0.4426
& 0.5897
& 0.8849
& 0.6508
& 0.4577
& 0.8587
\\

DINO-Base + RoseLoRA
& 16
& 0.5829
& 0.4125
& 0.3868
& 0.5194
& 0.8892
& 0.5663
& 0.4125
& 0.8344
\\

\addlinespace[1.5pt]
DINO-Base + RoSA
& 8
& 0.6456
& 0.4730
& 0.4657
& 0.5926
& 0.8992
& 0.6569
& 0.4730
& 0.8590
\\

\midrule
\textbf{DINO-Base + LoRSA (Ours)}
& \textbf{8}
& \textbf{0.6656}
& \textbf{0.5011}
& \textbf{0.5150}
& \textbf{0.6267}
& \textbf{0.9030}
& \textbf{0.6855}
& \textbf{0.5011}
& \textbf{0.8663}
\\

\bottomrule
\end{tabular}

\begin{tablenotes}[flushleft]
\footnotesize

\item The external test set contains 8,000 images: 852, 3,208, 2,876, and 1,064 images from density categories A, B, C, and D, respectively.

\end{tablenotes}
\end{threeparttable}
\end{table*}

\begin{table*}[t]
\centering
\caption{External evaluation of parameter-efficient adaptation methods
using the DINO-Base backbone for four-class breast density classification
on the RSNA dataset.\\}
\label{tab:rsna_external_dino_base}

\setlength{\tabcolsep}{3.2pt}
\renewcommand{\arraystretch}{1.14}
\scriptsize

\begin{threeparttable}
\begin{tabular}{
    @{}
    l
    c
    c c c c c c c c
    @{}
}
\toprule
\multirow{2}{*}{Method}
&
\multirow{2}{*}{$r$}
&
\multicolumn{8}{c}{External Test Performance}
\\
\cmidrule(lr){3-10}

&
&
\makecell{Acc.}
&
\makecell{Bal. Acc.}
&
\makecell{Macro-F1}
&
\makecell{Weighted-F1}
&
\makecell{Macro-AUROC}
&
\makecell{QWK}
&
\makecell{Sens.}
&
\makecell{Spec.}
\\
\midrule

DINO-Base (frozen)
& --
& 0.6782
& 0.4454
& 0.4307
& 0.6349
& 0.8315
& 0.6092
& 0.4454
& 0.8644
\\

\addlinespace[1.5pt]
DINO-Base + LoRA
& 8
& 0.6903
& 0.4402
& 0.4340
& 0.6470
& 0.8887
& 0.6262
& 0.4402
& 0.8690
\\

DINO-Base + LoRA
& 16
& 0.7093
& 0.5070
& 0.5211
& 0.6850
& 0.9061
& 0.6628
& 0.5070
& 0.8794
\\

\addlinespace[1.5pt]
DINO-Base + RoseLoRA
& 8
& 0.7024
& 0.4542
& 0.4515
& 0.6627
& 0.8871
& 0.6397
& 0.4542
& 0.8745
\\

DINO-Base + RoseLoRA
& 16
& 0.6525
& 0.4383
& 0.4449
& 0.6191
& 0.8893
& 0.5978
& 0.4383
& 0.8543
\\

\addlinespace[1.5pt]
DINO-Base + RoSA
& 8
& \underline{0.7280}
& \underline{0.5371}
& \underline{0.5536}
& \underline{0.7125}
& \underline{0.9074}
& \underline{0.7015}
& \underline{0.5371}
& \textbf{0.8881}
\\

\midrule
\textbf{DINO-Base + LoRSA (Ours)}
& \textbf{8}
& \textbf{0.7301}
& \textbf{0.5518}
& \textbf{0.5845}
& \textbf{0.7159}
& \textbf{0.9113}
& \textbf{0.7042}
& \textbf{0.5518}
& \underline{0.8876}
\\

\bottomrule
\end{tabular}

\begin{tablenotes}[flushleft]
\footnotesize

\item The external RSNA test set contains 29,470 images: 3,105, 12,651,
12,175, and 1,539 images from density categories A, B, C, and D,
respectively.

\end{tablenotes}
\end{threeparttable}
\end{table*}

\section{Conclusion}

This work examined a central limitation of low-rank parameter-efficient fine-tuning for biomedical domain generalization. A single low-rank update is efficient, but it also confines all downstream adaptation to one bounded pair of row and column subspaces. In complex medical tasks, that subspace must simultaneously represent globally shared task structure, localized corrections, and source-sensitive residual variation. We described this restriction as the low-rank cage. Sparsifying the same low-rank component does not remove the cage because the update remains rank constrained, and it may weaken the dense global transformation that standard LoRA captures effectively. RoSA correctly recognizes the need for low-rank and sparse components, but a support mask estimated once from a calibration subset cannot adapt as the dense component and task representation evolve.

LoRSA addresses this problem through a global--residual adaptation principle. Each update contains a dense low-rank component for globally coordinated specialization and a dynamically masked structured low-rank component for residual correction. These components are optimized jointly so that both the residual values and their structural allocation can evolve under the downstream objective. The practical use of LoRA-style and RoseLoRA-style factors is a direct realization of this principle, not the conceptual definition of the contribution.

The external experiments support the proposed direction. Using VinDr-Mammo as the source domain, LoRSA achieved the best macro-F1 on both MammosighTR and RSNA, outperforming the strongest baseline by 2.15 and 3.09 percentage points, respectively. Its advantage was more pronounced externally than internally, which is consistent with the goal of preserving transferable task structure rather than maximizing source-domain fit.

The weight-matrix complementarity analysis provides a more precise explanation of
this behavior. Across the adapted weight matrices, the dense
low-rank and structured-sparse components place approximately
\(92\%\) of their respective energies outside one another's
bilateral singular subspaces, indicating that the two components
learn largely complementary weight-update directions.

\bibliography{refs.bib}
\bibliographystyle{ieeetr}

\end{document}